%% file: main-realizable.tex
\documentclass[11pt]{article}
\usepackage[margin=1in]{geometry}

\usepackage{etoolbox}
\usepackage{macros}

\newtoggle{arxiv}
\toggletrue{arxiv}

\bluehyperref

\usepackage{microtype}
\usepackage{graphicx}
\usepackage{subfigure}
\usepackage{booktabs} 
\usepackage{makecell}

\usepackage[style=alphabetic,backend=bibtex,maxnames=12,maxbibnames=10,maxcitenames=10,maxalphanames=10,giveninits=true,doi=false,url=true]{biblatex}
\newcommand*{\citet}[1]{\AtNextCite{\AtEachCitekey{\defcounter{maxnames}{2}}} \textcite{#1}}

\newcommand*{\citep}[1]{\cite{#1}}

\bibliography{bib}
\let\citealp\citep

\usepackage{hyperref}
\usepackage[capitalize]{cleveref}
\usepackage{crossreftools}
\usepackage{tikz}
\usepackage{pgfplots}

\usepackage{cellspace,                  
	tabularx}

\usepackage[noend]{algorithmic}
\usepackage{algorithm}

\usepackage{comment}

\usepackage{cellspace, tabularx}                 

\newcolumntype{C}{>{\centering\arraybackslash}X}
\addparagraphcolumntypes{C}
\newcolumntype{P}[1]{>{\arraybackslash}p{#1}}
\usepackage{array}
\usepackage{multirow}

\newcolumntype{x}[1]{%
	>{\raggedleft\hspace{0pt}}p{#1}}%
\usepackage{soul}

\newif\ifcomments
\commentsfalse
\ifcomments
\newcommand{\ha}[1]{
		\textcolor{blue}{\textbf{HA:} {#1}}
}
\newcommand{\tk}[1]{
		\textcolor{magenta}{\textbf{TK:} {#1}}
}
\newcommand{\kt}[1]{
		\textcolor{red}{\textbf{KT:} {#1}}
}
\newcommand{\vf}[1]{
		\textcolor{green}{\textbf{VF:} {#1}}
}
\else
\newcommand{\ha}[1]{}
\newcommand{\tk}[1]{}
\newcommand{\kt}[1]{}
\newcommand{\vf}[1]{}
\fi

\newcommand{\ifrac}[2]{{#1}/{#2}}

\usepackage{accents}
\newcommand{\ubar}[1]{\underaccent{\bar}{#1}}\usepackage{accents}

\title{Near-Optimal Algorithms for Private Online Optimization \\ in the Realizable Regime}
\author{%
    Hilal Asi\thanks{Apple; \texttt{hilal.asi94@gmail.com} } 
    \and Vitaly Feldman\thanks{Apple; \texttt{vitaly.edu@gmail.com}.}
    \and Tomer Koren\thanks{Blavatnik School of Computer Science, Tel Aviv University; \texttt{tkoren@tauex.tau.ac.il}.}
    \and Kunal Talwar\thanks{Apple; \texttt{kunal@kunaltalwar.org}.}
    }

\begin{document}

\maketitle

\input{realizable-abstract}

\input{realizable-introduction}

\input{realizable-pre}

\input{realizable-ub-experts}

\input{realizable-ub-OCO}

\input{realizable-LB.tex}

\input{conclusion.tex}

\subsection*{Acknowledgements}

This work has received support from the Israeli Science Foundation (ISF) grant no.~2549/19 and the Len Blavatnik and the Blavatnik Family foundation.

\printbibliography

\appendix

\input{appendix-chernoff}

\end{document}

%% file: realizable-abstract.tex
\begin{abstract}
    We consider online learning problems in the realizable setting, where there is a zero-loss solution, and propose new Differentially Private (DP) algorithms that obtain near-optimal regret bounds.
    For the problem of online prediction from experts, we design new algorithms 
    that obtain near-optimal regret $\wt O \big( \diffp^{-1} \log^{1.5}{d} \big)$ where $d$ is the number of experts. 
    This significantly improves over the best existing regret bounds for the DP non-realizable setting which are $\wt O \big( \diffp^{-1} \min\big\{d, T^{1/3}\log d\big\} \big)$. 
    We also develop an adaptive algorithm for the small-loss setting with regret $O(L\opt \log d + \diffp^{-1} \log^{1.5}{d})$ where $L\opt$ is the total loss of the best expert.
    Additionally, we consider DP online convex optimization in the realizable setting and propose an algorithm with near-optimal regret $\wt O \big(\diffp^{-1} d^{1.5} \big)$,
    as well as an algorithm for the smooth case with regret $\wt O \big( \diffp^{-2/3} (dT)^{1/3} \big)$, both significantly improving over existing bounds in the non-realizable regime.
\end{abstract}

%% file: realizable-introduction.tex
\section{Introduction}

We study the problem of private online optimization in the realizable setting where there is a zero-loss solution. In this problem, an online algorithm $\A$ interacts with an adversary over $T$ rounds. The adversary picks a (non-negative) loss function $\ell_t: \mc{X} \to \R$ at round $t$ and simultaneously the algorithm $\A$ picks a response $x_t$, suffering loss $\ell_t(x_t)$. The algorithm aims to minimize the regret, which is the loss compared to the best solution $x\opt \in \mc{X}$ in hindsight, while at the same time keeping the sequence of predictions $x_1,\ldots,x_T$ differentially private with respect to individual loss functions. 

In this paper, we focus on two well-studied instances of this problem. In differentially private online prediction from experts (DP-OPE), we have $d$ experts $\mc{X} = [d]$ and the adversary chooses a loss function $\ell_t: [d] \to [0,1]$. Our second setting is differentially private online convex optimization (DP-OCO) where $\mc{X} \subset \R^d$ is a convex set with bounded diameter, and the adversary chooses convex and $L$-Lipschitz loss functions $\ell_t : \mc{X} \to \R^+$.

Several papers have recently studied DP-OPE and DP-OCO in the general non-realizable setting~\cite{JainKoTh12,SmithTh13,JainTh14,AgarwalSi17,KairouzMcSoShThXu21}. These papers have resulted in different algorithms with sub-linear regret for both problems. For DP-OPE, ~\citet{AgarwalSi17,JainTh14} developed private versions of follow-the-regularized-leader (FTRL) obtaining (normalized) regret $\min\big\{ {d}/{T\diffp}, {\sqrt{T \log d}}/{T\diffp} \big\}$. 
More recently, \citet{AsiFeKoTa22} developed low-switching algorithms for DP-OPE with oblivious adversaries, obtaining normalized regret roughly $O(\sqrt{\log (d)/T} +  \log d/T^{2/3} \eps)$.
Additionally, for the problem of DP-OCO, \citet{KairouzMcSoShThXu21} have recently proposed a DP-FTRL algorithm based on the binary tree mechanism which obtains (normalized) regret $ \big({\sqrt{d}}/{T\diffp} \big)^{1/2}$.

Despite this progress, the regret bounds of existing algorithms are still polynomially worse than existing lower bounds. Currently, the only existing lower bounds for oblivious adversaries are the trivial bounds from the non-online versions of the same problems: for DP-OPE, lower bounds for private selection~\cite{SteinkeUll17b} imply a (normalized) regret lower bound of $O({\log(d)}/{T\diffp)}$, while existing lower bounds for DP-SCO~\cite{FeldmanKoTa20} give a (normalized) regret lower bound of $\Omega({\sqrt{d}}/{T\diffp})$ for DP-OCO.

Practical optimization problems arising from over-parameterized models often lead to instances that additionally satisfy {\em realizability}, i.e. that the optimal loss is zero or close to zero. This motivates the study of designing algorithms that can do better under this assumption.  Realizability has been studied since the early days of learning theory and ubiquitous in the non-private online optimization literature~\cite{SrebroSrTe10,Shalev12,Hazan16}.
It has proven useful for improving regret bounds in non-private OPE and OCO~\cite{Shalev12,SrebroSrTe10} and in the closely related problem of differentially private stochastic convex optimization (DP-SCO)~\cite{AsiChChDu22}.  In this work we study DP-OPE and DP-OCO in the realizable setting and develop new algorithms that obtain near-optimal regret bounds in several settings.





\subsection{Contributions}
We propose new algorithms and lower bounds 
for the problems of differentially private online prediction from experts (DP-OPE) and differentially private online convex optimization (DP-OCO) in the realizable setting. The following are our primary contributions:

\begin{itemize} 
\item \textbf{Near-optimal algorithms for DP-OPE.}~~
We design new algorithms that obtain near-optimal regret $\wt O \left( \log^{1.5}(d)/\diffp \right)$ for DP-OPE with $d$ experts when there is a zero-loss expert. 
The best existing algorithms for non-realizable DP-OPE obtain significantly worse regret bounds $\min\big\{{d}/{\diffp}, T^{1/3} \log d/{\diffp} \big\}$~\cite{AgarwalSi17,AsiFeKoTa22}, which have a polynomial dependence on either $T$ or the number of experts $d$.
Our algorithms build on sequential applications of the exponential mechanism to pick a good expert, and the sparse-vector-technique to identify when the current expert is no longer a good expert (with near-zero loss).  
Crucially, an oblivious adversary cannot identify which expert the algorithm has picked, resulting in a small number of switches.
We deploy a potential-based proof strategy to show that this algorithm have logarithmic number of switches.
We also show that a lower bound of $\Omega(\log d / \diffp)$ holds for any $\diffp$-DP algorithm even in the realizable case.
\item \textbf{Adaptive algorithms for DP-OPE with low-loss experts.}~~
We also develop an algorithm that adapts to the setting where there is an expert with low loss, that is, $L\opt = \min_{x \in [d]} \sum_{t=1}^T \ell_t(x)$. Our algorithms are adaptive to the value of $L\opt$ and obtain total regret of $L\opt \log d + \diffp^{-1} \log^{1.5} d$.
\item \textbf{Near-optimal regret for low-dimensional DP-OCO.}~~
Building on our algorithms for DP-OPE, we propose a new algorithm for DP-OCO that obtains regret $\wt O \left( d^{1.5}/\diffp \right)$. This is near-optimal for low-dimensional problems where $d=O(1)$ and improves over the best existing algorithm which obtains a normalized regret $ (\sqrt{d}/T\diffp)^{1/2}$~\cite{KairouzMcSoShThXu21}.
\item \textbf{Improved regret for smooth DP-OCO.}~~
When the loss function is smooth, we show that DP-FTRL~\cite{KairouzMcSoShThXu21} with certain parameters obtains an improved normalized regret of 
$(\sqrt{d}/T\diffp)^{2/3}$ if there is a zero-loss expert.
\end{itemize}

\begin{table*}[t]
\begin{center}
		\begin{tabular}{| Sc | Sc | Sc |}
		    \hline
			  & \textbf{\darkblue{Non-realizable}} & \makecell{\textbf{\darkblue{Realizable}}\\\textbf{(This work)}}\\
			\hline
			{{\textbf{DP-OPE}}} & $\displaystyle \min\left\{ \frac{\sqrt{d}}{T\diffp}, \sqrt{\frac{ \log d}{T}} + \frac{\log d}{T^{2/3}\eps} \right\}$~\footnotesize{\cite{AgarwalSi17,AsiFeKoTa22}}  & $\displaystyle \frac{\log^{1.5} d}{T\diffp}$ \\
			\cline{1-3} 
			\textbf{{DP-OCO}} & $\displaystyle  \left(\frac{\sqrt{d}}{T\diffp} \right)^{1/2}$~\footnotesize{\cite{KairouzMcSoShThXu21} } & $\displaystyle \frac{d^{1.5}}{T\diffp}$ \\ 
			\cline{1-3}
			\textbf{{DP-OCO (smooth)}} & $\displaystyle  \left(\frac{\sqrt{d}}{T\diffp} \right)^{1/2}$~\footnotesize{\cite{KairouzMcSoShThXu21}} & $\displaystyle \left(\frac{\sqrt{d}}{T\diffp} \right)^{2/3}$ \\
			\hline
		\end{tabular}
     \end{center}
          \caption{Comparison between (normalized) regret upper bounds for the realizable and non-realizable case for both DP-OPE and DP-OCO. For readability, we omit logarithmic factors in $T$ and $1/\delta$.}
     \label{tab:temps}
\end{table*}

\subsection{Related work}
Several works have studied online optimization in the realizable setting, developing algorithms with better regret bounds~\cite{Shalev12,SrebroSrTe10}. For online prediction from experts, the weighted majority algorithm obtains a regret bound of $4\log{d}$ compared to $O(\sqrt{T\log d})$ in the non-realizable setting. Moreover, for online convex optimization, \citet{SrebroSrTe10} show that online mirror descent achieves regret $4\beta D^2 + 2\sqrt{\beta D^2 T L\opt}$ compared to $O(\sqrt{T})$ in the general case.


On the other hand, the private online optimization literature has mainly studied the general non-realizable case~\cite{JainKoTh12,SmithTh13,JainTh14,AgarwalSi17,KairouzMcSoShThXu21}.
For online prediction from experts, the best existing regret bounds for \ed-DP are $O(\diffp^{-1}\sqrt{T \log d \log(1/\delta)})$~\cite{JainTh14}  and $O(\sqrt{T\log d} + \diffp^{-1} \sqrt{d \log(1/\delta)} \log d \log^2 T)$~\cite{AgarwalSi17}. 
\citet{AsiFeKoTa22} show that these rates can be improved using a private version of the shrinking dartboard algorithm, obtaining regret roughly $O(\sqrt{T \log d} + T^{1/3} \log d/\eps)$.
For online convex optimization, \citet{KairouzMcSoShThXu21} developed a private follow-the-regularized-leader algorithm using the binary tree mechanism that obtains normalized regret bound $\wt O\big( {\sqrt{d}}/{T\diffp} \big)^{1/2}$. 

The realizable setting has recently been studied in the different but related problem of differentially private stochastic convex optimization (DP-SCO)~\cite{AsiChChDu22}.  
DP-SCO and DP-OCO are closely related as one can convert an OCO algorithm into an SCO algorithm using standard online-to-batch transformations~\cite{Hazan16} 
\citet{AsiChChDu22} study DP-SCO problems in the interpolation regime where there exists a minimizer that minimizes all loss functions, and propose algorithms that improve the regret over the general setting if the functions satisfy certain growth conditions.





%% file: realizable-pre.tex
\section{Preliminaries}
In online optimization, we have an interactive $T$-round game between an adversary and an online algorithm. In this paper, we focus on oblivious adversaries that choose in advance a sequence of loss functions $\ell_1,\dots,\ell_T$ where $\ell_t : \mc{X} \to \R$. Then, at round $t$, the adversary releases a loss function $\ell_t$ and simultaneously the algorithm plays a solution $x_t \in \mc{X}$. The algorithm then suffers loss $\ell_t(x_t)$ at this round. The regret of the online algorithm is
\begin{equation*}
    \reg_T(\A) = \sum_{t=1}^T \ell_t(x_t) - \min_{x^\star \in \mc{X}} \sum_{t=1}^T \ell_t(x^\star).
\end{equation*}
For ease of notation, for an oblivious adversary that chooses a loss sequence $\Ds = (\ell_1,\dots,\ell_T) $, we let $\A(\Ds) = (x_1,\dots,x_T)$ denote the output of the interaction between the online algorithm and the adversary.

In this work, we are mainly interested in two instances of the above general online optimization problem:

\begin{itemize} 
\item \textbf{Online prediction from experts (OPE).}~~
In this problem, we have a set of $d$ experts $\mc{X} = [d]$, and the adversary chooses loss functions $\ell_t : [d] \to [0,1]$.
\item\textbf{Online convex optimization (OCO).}~~
In OCO, we are optimizing over a convex set $\mc{X} \subseteq \R^d$ with bounded diameter $\diam(\mc{X}) \le D$,%
\footnote{The diameter of a set $\mc{X} \subseteq \R^d$ (in Euclidean geometry) is defined as $\diam(\mc{X}) = \sup_{x,y \in \mc{X}} \|x-y\|$.}
and the adversary chooses loss functions $\ell_t : \mc{X} \to \R$ that are convex and $L$-Lipschitz. 
\end{itemize}

We are mainly interested in the so-called realizable setting. More precisely, we say than an OPE (or OCO) problem is \emph{realizable} if there exists a feasible solution $x\opt \in \mc{X}$ such that $L\opt = \sum_{t=1}^T \ell_t(x\opt) = 0$. We also extend some of our results to the near-realizable setting where $0 < L\opt \ll T$.

The main goal of this paper is to study both of these problems under the restriction of differential privacy.

\begin{definition}[Differential Privacy]
\label{def:DP}
	A randomized  algorithm $\A$ is \emph{\ed-differentially private} against oblivious adversaries (\ed-DP) if, for all sequences $\Ds=(\ell_1,\dots,\ell_T)$ and $\Ds'=(\ell'_1,\dots,\ell'_T)$ that differ in a single element, and for all events $\cO$ in the output space of $\A$, we have
	\[
	\Pr[\A(\Ds)\in \cO] \leq e^{\eps} \Pr[\A(\Ds')\in \cO] +\delta.
	\]
\end{definition}
We note that our algorithms satisfy a stronger privacy guarantee against adaptive adversaries (see for example the privacy definition in~\cite{JainRaSiSm21}). However, we choose to focus solely on oblivious adversaries for ease of presentation and readability.

\subsection{Background on Differential Privacy}
\newcommand{\AbThr}{\ensuremath{\mathsf{AboveThreshold}}}
\newcommand{\init}{\ensuremath{\mathsf{InitializeSparseVec}}} 
\newcommand{\addq}{\ensuremath{\mathsf{AddQuery}}} 
\newcommand{\test}{\ensuremath{\mathsf{TestAboThr}}} 

In our analysis, we require the following standard privacy composition result.
\begin{lemma}[Advanced composition~\citealp{DworkRo14}] 
\label{lemma:advanced-comp}
    If $\A_1,\dots,A_k$ are randomized algorithms that each is $(\diffp,\delta)$-DP, then their composition $(\A_1(\Ds),\dots,A_k(\Ds))$ is $(\sqrt{2k \log(1/\delta')} \diffp + k \diffp (e^\diffp - 1),\delta' + k \delta)$-DP.
\end{lemma}

In addition to basic facts about differential privacy such as composition and post-processing, our development uses two key techniques from the privacy literature: the Sparse-vector-technique and the binary tree mechanism, which we now describe.

\paragraph{Sparse vector technique.}

We recall the sparse-vector-technique~\cite{DworkRo14} which we use for the realizable setting in~\cref{sec:upper-bounds-realizable}. Given an input $\Ds = (z_1,\dots,z_n) \in \domain^n$, the algorithm takes a stream of queries $q_1,q_2,\dots,q_T$ in an online manner. We assume that each $q_i$ is $1$-sensitive, that is, $|q_i(\Ds) - q_i(\Ds') | \le 1$ for neighboring datasets $\Ds,\Ds' \in \domain^n$ that differ in a single element.
We have the following guarantee. 
\begin{lemma}[\citealp{DworkRo14}, Theorem 3.24]
\label{lemma:svt}
    Let $\Ds = (z_1,\dots,z_n) \in \domain^n$.
    For a threshold $L$ and $\beta>0$, there is an $\diffp$-DP algorithm (\AbThr) that halts at time $k \in [T+1]$ such that for $\alpha = \frac{8(\log T + \log(2/\beta))}{\diffp}$ with probability at least $1-\beta$,
    \begin{itemize} 
        \item For all $t < k$, $q_i(\Ds) \le L + \alpha$;
        \item $q_k(\Ds) \ge L - \alpha$ or $k = T+1$.
    \end{itemize}
\end{lemma}

To facilitate the notation for using \AbThr~in our algorithms, we assume that it has the following components:
\begin{enumerate}
    \item $\init(\diffp,L,\beta)$: initializes a new instance of \AbThr~with privacy parameter $\diffp$, threshold $L$, and probability parameter $\beta$. This returns an instance (data structure) $Q$ that supports the following two functions.
    \item $Q.\addq(q)$: adds a new query $q:\domain^n \to \R$ to $Q$.
    \item $Q.\test()$: tests if the last query that was added to $Q$ was above threshold. In that case, the algorithm stops and does not accept more queries.
\end{enumerate}


\newcommand{\BinTr}{BinaryTree}

\paragraph{The binary tree mechanism.}

We also build on the binary tree mechanism~\cite{DworkNaPiRo10,ChanShSo11} which allows to privately estimate the running sum of a sequence of $T$ numbers $a_1,\dots,a_T \in [0,1]$. 
\begin{lemma}[\citealp{DworkNaPiRo10}, Theorem 4.1]
\label{lemma:bt}
    Let $\diffp \le 1$.  There is an $\diffp$-DP algorithm (\BinTr) that takes a stream of numbers $a_1,a_2,\dots,a_T$ and outputs $c_1,c_2,\dots,c_T$ such that for all $t \in [T]$ with probability at least $1-\beta$,
    \begin{equation*}
        \Big| c_t - \sum_{i=1}^t a_i \Big| 
        =
        \frac{1}{\diffp} \cdot \mathsf{poly}(\log(\beta^{-1})\log{T}).
    \end{equation*}
\end{lemma}
The same approach extends to the case when $a_i$'s are vectors in $\mathbb{R}^d$ with $\|a_i\|_2 \leq 1$. In this case, the error vector $(c_t - \sum_{i=1}^t a_i)$ is distributed at $\mathcal{N}(0, d \cdot \mathsf{poly}(\log T/\beta\delta)/\diffp^2 \mathbb{I})$ and the mechanism satisfies $(\diffp,\delta)$-DP.

\paragraph{Additional notation.} 

For a positive integer $k \in \N$, we let $[k] = \{1,2,\dots,k\}$. Moreover, for a sequence $a_1,\dots,a_t$, we use the shorthand $a_{1:t} = a_1,\dots,a_t$.

%% file: realizable-ub-experts.tex
\section{Near-optimal regret for online prediction from experts}
\label{sec:upper-bounds-realizable}
In this section, we consider the online prediction from experts problem in the near-realizable regime,
where the best expert achieves small loss $L\opt ll T$. 
Under this setting, 
we develop a new private algorithm that achieves regret $\wt O(L\opt \log d + \log^{3/2}(d)/\diffp)$. For the realizable setting where $L\opt=0$, this algorithm obtains near-optimal regret $\wt O(\log^{3/2} (d)/\diffp)$. 

The algorithm builds on the fact that an oblivious adversary cannot know which expert the algorithm picks. Therefore, if the algorithm picks a random good expert with loss smaller than $L\opt$, the adversary has to increase the loss for many experts before identifying the expert chosen by the algorithm. The algorithm will therefore proceed as follows:
at each round, privately check using sparse-vector-technique whether the previous expert is still a good expert (has loss nearly $L\opt$). If not, randomly pick (privately) a new expert from the set of remaining good experts. The full details are in~\cref{alg:SVT-zero-loss}. 

The following theorem summarizes the performance of~\cref{alg:SVT-zero-loss}. 
\begin{theorem}
\label{thm:ub-realizable}
    Let $\ell_1,\dots,\ell_T \in [0,1]^d$ be chosen by an oblivious adversary such that there is $x\opt \in [d]$ such that $\sum_{t=1}^T \ell_t(x\opt) \le L\opt$. Let $0 < \beta < 1/2$, $B = \log(2T^2/\beta)$,  $K = 6\ceil{\log d} +  24 \log(1/\beta)$, and $L = L\opt + 4/\eta + \frac{8B}{\diffp}$. If $\eta = \diffp/2K $ then
    \cref{alg:SVT-zero-loss} is $\diffp$-DP and with probability at least $1-O(\beta)$ has regret 
    \begin{equation*}
        \sum_{t=1}^T \ell_t(x_t) \le O\left( L\opt \log(d/\beta) + \frac{\log^2(d) + \log(T/\beta) \log(d/\beta)}{\diffp}  \right).
    \end{equation*}
    Further, if $\diffp \le  \sqrt{\log T \log(1/\delta)}$ and $\eta = \diffp/4\sqrt{2K \log(1/\delta)}$ then
    \cref{alg:SVT-zero-loss} is $(\diffp,\delta)$-DP and with probability at least $1-O(\beta)$ has regret 
    \begin{equation*}
        \sum_{t=1}^T \ell_t(x_t) \le O\left( L\opt \log(d/\beta) + \frac{\log^{3/2}(d)\sqrt{\log(1/\delta)} + \log(T/\beta) \log(d/\beta)}{\diffp}  \right).
    \end{equation*}
\end{theorem}
While \cref{alg:SVT-zero-loss} requires the knowledge of $L\opt$, we also design an adaptive version that does not require $L\opt$ in the next section. Note that the algorithm obtains regret roughly $\log^{3/2} (d)/\diffp$ for the realizable setting where $L\opt = 0$.

\begin{algorithm}[t]
	\caption{Sparse-Vector for zero loss experts }
	\label{alg:SVT-zero-loss}
	\begin{algorithmic}[1]
		\REQUIRE Switching bound $K$, optimal loss $L\opt$, Sampling parameter $\eta$, Threshold parameter $L$, failure probability $\beta$, privacy parameters $(\diffp,\delta)$
		\STATE Set $k=0$ and current expert $x_0 = \mathsf{Unif}[d]$	
        \STATE Set $t_p = 0$
        \WHILE{$t \le T$\,}
            \STATE Set $ x_{t} =  x_{t-1}$
            \IF{$k < K$}
                \STATE $Q = \init(\diffp/2,L,\beta/T)$
                \WHILE{Q.\test() = False}
                    \STATE Set $ x_{t} =  x_{t-1}$
                    \STATE Define a new query $q_t = \sum_{i=t_p}^{t-1} \ell_i(x_t)$
                    \STATE Add new query $Q.\addq(q_t)$
                    \STATE Receive loss function $\ell_t: [d] \to [0,1]$
        	        \STATE Pay cost $\ell_t(x_t)$
        	        \STATE Update $t = t+1$ 
                \ENDWHILE
            
                \STATE Sample $x_t$ from the exponential mechanism with scores $s_t(x) = \max \left(\sum_{i=1}^{t-1} \ell_i(x),L\opt \right)$ for $x \in [d]$:
                \begin{equation*}
                    \P(x_t = x) \propto e^{-\eta s_t(x)/2 }
                \end{equation*}
                \STATE Set $k = k + 1$ and $t_p = t$
            \ENDIF
            \STATE Receive loss function $\ell_t: [d] \to [0,1]$
            \STATE Pay cost $\ell_t(x_t)$ 
            \STATE Update $t = t+1$
        \ENDWHILE
	\end{algorithmic}
\end{algorithm}

\begin{proof}
First, we prove the privacy guarantees of the algorithm using privacy composition results: there are $K$ applications of the exponential mechanism with privacy parameter $\eta$. Moreover, sparse-vector is applied over each user's data only once, hence the $K$ applications of sparse-vector are $\diffp/2$-DP. Overall, the algorithm is $(\diffp/2 + K\eta)$-DP and $(\diffp/2 + \sqrt{2K \log(1/\delta)} \eta + K \eta (e^\eta - 1),\delta)$-DP (using advanced compositions; see~\Cref{lemma:advanced-comp}). Setting $\eta = \diffp/2K$ results in $\diffp$-DP and $\eta = O(\diffp/\sqrt{K \log(1/\delta)})$ results in \ed-DP.

We proceed to analyze utility. First, note that the guarantees of the sparse-vector algorithm (\Cref{lemma:svt}) imply that with probability at least $1-\beta$ for each time-step $t \in [T]$, if sparse-vector identifies above threshold query then $s_t(x) \ge \ubar \Delta \defeq L - \frac{8B}{\diffp} \ge 4/\eta$. Otherwise, $s_t(x) \le \bar \Delta \defeq L + \frac{8B}{\diffp}$. In the remainder of the proof, we condition on this event. The idea is to show that the algorithm has logarithmic number of switches, and each switch the algorithm pays roughly $1/\diffp$ regret.
\\
To this end, we define a potential at time $t \in [T]$:
\begin{equation*}
    \phi_t = \sum_{x \in [d]} e^{-\eta L_t(x)/2} ,
\end{equation*}
where $L_t(x) = \max( \sum_{j=1}^{t-1} \ell_j(x), L\opt)$.
Note that $\phi_1 = d e^{-\eta L\opt/2}$ and $\phi_t \ge e^{-\eta L\opt/2}$ for all $t \in [T]$ as there is $x \in [d]$ such that $\sum_{t=1}^T \ell_t(x) = L\opt$.
We split the iterates to $m = \ceil{\log d}$ rounds $t_0,t_1,\dots,t_m$ where $t_i$ is the largest $t\in[T]$ such that $\phi_{t_i} \ge \phi_1/2^{i}$. 
Let $Z_i$ be the number of switches in $[t_i,t_{i+1}-1]$ (number of times the exponential mechanism is used to pick $x_t$).
The following key lemma shows that $Z_i$ cannot be too large.
\begin{lemma}
\label{lemma:prob-switch}
    Fix $0 \le i \le m-1$. Then for any $1 \le k \le T$, it holds that
    \begin{equation*}
        P(Z_i = k+1) \le (2/3)^k.
    \end{equation*}
\end{lemma}
\begin{proof}
Let $t_i \le t \le t_{i+1}$ be a time-step where a switch happens (exponential mechanism is used to pick $x_{t}$). Note that $\phi_{t_{i+1}} \ge \phi_{t}/2$. We prove that the probability that $x_t$ is switched between $t$ and $t_{i+1}$ is at most $2/3$. To this end, note that if $x_t$ is switched before $t_{i+1}$ then $\sum_{i=t}^{t_{i+1}} \ell_i(x) \ge \ubar \Delta$ as sparse-vector identifies $x_t$, and therefore $L_{t_{i+1}}(x) - L_t(x) \ge \ubar \Delta - L\opt \ge 4/\eta$. 
Thus we have that
\begin{align*}
P(x_{t} \text{ is switched before $t_{i+1}$})
    & \le \sum_{x \in [d]} P(x_t=x) \indic{L_{t_{i+1}}(x) - L_t(x) \ge 4/\eta} \\
    & = \sum_{x \in [d]} \frac{e^{-\eta L_{t}(x)/2}}{\phi_{t}}\cdot \indic{L_{t_{i+1}}(x) - L_t(x) \ge 4/\eta } \\
    & \le \sum_{x \in [d]} \frac{e^{-\eta L_{t}(x)/2}}{\phi_{t}}\cdot \frac{1 - e^{-\eta (L_{t_{i+1}}(x)- L_t(x))/2}}{1 - e^{-2} } \\
    & \le 4/3 (1 - \phi_{t_{i+1}}/\phi_{t}) \\
    & \le 2/3.
\end{align*}
where the second inequality follows the fact that $\indic{a \ge b} \le \frac{1-e^{-\eta b}}{1 - e^{-\eta a}}$ for $a,b,\eta \ge 0$, 
and the last inequality since $\phi_{t_{i+1}}/\phi_{t_1} \ge 1/2$.
This argument shows that after the first switch inside the range $[t_i,t_{i+1}]$, each additional switch happens with probability at most $2/3$. The claim follows.
\end{proof}

We now proceed with the proof. Let $Z = \sum_{i=0}^{m-1} Z_i$ be the total number of switches. 
Note that $Z \le m + \sum_{i=0}^{m-1} \max(Z_i-1,0)$ and \Cref{lemma:prob-switch} implies $\max(Z_i-1,0)$ is upper bounded by a geometric random variable with success probability $1/3$. Therefore, using concentration of geometric random variables (\Cref{lemma:geom-concentration}), we get that
\begin{equation*}
 P(Z \ge  6m + 24 \log(1/\beta) ) \le \beta . 
\end{equation*}
Noting that $K \ge 6m + 24 \log(1/\beta)$, this shows that the algorithm does not reach the switching budget with probability $1-O(\beta)$. Thus, the guarantees of the sparse-vector algorithm imply that the algorithm pays regret at most $\bar \Delta$ for each switch, hence the total regret of the algorithm is at most $O(\bar \Delta (m + \log(1/\beta))) = O(\bar \Delta \log(d/\beta)) $. The claim follows as $\bar \Delta \le L\opt + 4/\eta + 16B/\diffp$.
\end{proof}

\subsection{Adaptive algorithms for DP experts}
While~\cref{alg:SVT-zero-loss} achieves near-optimal loss for settings with low-loss experts, it requires the knowledge of the value of $L\opt$. As $L\opt$ is not always available in practice, our goal in this section is to develop an adaptive version of~\cref{alg:SVT-zero-loss} which obtains similar regret without requiring the knowledge of $L\opt$. Similarly to other online learning problems, we propose to use the doubling trick~\cite{KalaiVe05} to design our adaptive algorithms. We begin with an estimate $L\opt_1 = 1$ of $L\opt$. Then we apply~\cref{alg:SVT-zero-loss} using $L\opt = L\opt_1$ until the exponential mechanism picks an expert that contradicts the current estimate of $L\opt$, that is,  $\sum_{i=1}^{t-1} \ell_i(x_t) \gg L\opt_1$. We use the Laplace mechanism to check this privately. Noting that this happens with small probability if $L\opt \le L\opt_1$, we conclude that our estimate of $L\opt$ was too small and set a new estimate $L\opt_2 = 2 L\opt_1$ and repeat the same steps. As $L\opt \le T$, this process will stop in at most $\log T$ phases, hence we can divide the privacy budget equally among phases while losing at most a factor of $\log T$. We present the full details in~\Cref{alg:SVT-ada}.

\begin{algorithm}[t]
	\caption{Adaptive Sparse-Vector for low-loss experts }
	\label{alg:SVT-ada}
	\begin{algorithmic}[1]
		\REQUIRE Failure probability $\beta$
            \STATE Set $\diffp_0 = \diffp / 2\log T$
		\STATE $K = \log d + 2 \log T/\beta$, $\eta = \diffp_0/2K$, $B = \log T + \log(2T/\beta)$	
            \STATE Set $\bar L\opt = 1$, $L = L\opt + 4/\eta + \frac{8B}{\diffp_0}$
        \WHILE{$t < T$}
            \STATE Run~\Cref{alg:SVT-zero-loss} with parameters $K$, $\bar  L\opt$, $\eta$,  $L$, $\beta$, $\diffp_0$
            \IF{\Cref{alg:SVT-zero-loss} applies the exponential mechanism (step 12)}
                \STATE Calculate $\bar L_t = \sum_{i=1}^{t-1} \ell_i(x_t) + \zeta_t$ where $\zeta_t \sim \lap(K/\diffp_0)$
                \IF{$\bar L_t  > \bar  L\opt - 5K\log(T/\beta)/\diffp_0 $}
                    \STATE Set $\bar  L\opt = 2 \bar 
 L\opt$
                    \STATE Go to step 4
                \ENDIF
            \ENDIF
        \ENDWHILE
	\end{algorithmic}
\end{algorithm}

We have the following guarantees for the adaptive algorithm.
\iftoggle{arxiv}{}{
We defer the proof to~\Cref{sec:proof-ub-realizable-ada}.
}
\begin{theorem}
\label{thm:ub-realizable-ada}
    Let $\ell_1,\dots,\ell_T \in [0,1]^d$ be chosen by an oblivious adversary such that there is $x\opt \in [d]$ such that $\sum_{t=1}^T \ell_t(x\opt) \le L\opt$. Let $0 < \beta < 1/2$. Then
    \Cref{alg:SVT-ada} is $\diffp$-DP and with probability at least $1-O(\beta)$ has regret 
    \begin{equation*}
        \sum_{t=1}^T \ell_t(x_t) \le O\left(  L\opt \log(d/\beta) \log(T) + \frac{\log^2(d)\log(T) + \log(T/\beta) \log(d/\beta) \log(T)}{\diffp}  \right).
    \end{equation*}
\end{theorem}
\iftoggle{arxiv}{
\begin{proof}
    First we prove privacy. Note that $\bar L\opt$ can change at most $\log(T)$ times as $L\opt \le T$. Therefore, we have at most $\log(T)$ applications of~\Cref{alg:SVT-zero-loss}. Each one of these is $\diffp/(2\log(T))$-DP. Moreover, since we have at most $K$ applications of the exponential mechanism in~\Cref{alg:SVT-zero-loss}, we have at most $K \log(T)$ applications of the Laplace mechanism in~\Cref{alg:SVT-ada}. Each of these is $\diffp/2K\log(T)$-DP. Overall, privacy composition implies that the final privacy is $\diffp$-DP.

    Now we prove utility. \Cref{alg:SVT-ada} consists of at most $\log(T)$ applications of~\Cref{alg:SVT-zero-loss} with different values of $\bar L\opt$. We will show that each of these applications incurrs low regret.
    Consider an application of~\Cref{alg:SVT-zero-loss} with $\bar L\opt$. If $\bar L\opt \ge L\opt$, then~\Cref{thm:ub-realizable} implies that the regret is at most $$O\left( \bar L\opt \log(d/\beta) + \frac{\log^2(d) + \log(T/\beta) \log(d/\beta)}{\diffp_0}  \right).$$ 
    Now consider the case where $\bar L\opt \le L\opt$. We will show that~\Cref{alg:SVT-ada} will double $\bar L\opt$ and that the regret of~\Cref{alg:SVT-zero-loss} up to that time-step is not too large.
    Let $t_0$ be the largest $t$ such that $\min_{x \in [d]} \sum_{t=1}^{t_0} \ell_t(x) \le \bar L\opt$. Note that up to time $t_0$, the best expert had loss at most $\bar L\opt$ hence the regret up to time $t_0$ is $$O\left( \bar L\opt \log(d/\beta) + \frac{\log^2(d) + \log(T/\beta) \log(d/\beta)}{\diffp_0}  \right).$$ 
    Now let $t_1$ denote the next time-step when~\Cref{alg:SVT-zero-loss} applies the exponential mechanism. Sparse-vector guarantees that in the range $[t_0,t_1]$ the algorithm suffers regret at most $O\left( \bar L\opt  + \frac{\log(d) + \log(T/\beta) }{\diffp_0}  \right)$. Moreover, the guarantees of the Laplace mechanism imply that at this time-step, $\bar L_t \ge \bar  L\opt - 5K\log(T/\beta)/\diffp_0$ with probability $1-\beta$, hence~\Cref{alg:SVT-ada} will double $\bar L\opt$ and run a new application of~\Cref{alg:SVT-zero-loss}. Overall, an application of~\Cref{alg:SVT-zero-loss} with $\bar L\opt \le L\opt$ results in regret $(L\opt + \frac{1}{\diffp_0}) \cdot \mathsf{poly} (\log \frac{Td}{\beta})$ and doubles $\bar L\opt$. Finally, note that if $\bar L\opt \ge L\opt + 5\log(T/\beta)/\diffp_0$ then with probability $1-\beta$ the algorithm will not double the value of $\bar L\opt$. As each application of ~\Cref{alg:SVT-zero-loss} has regret $$O\left( \bar L\opt \log(d/\beta) + \frac{\log^2(d) + \log(T/\beta) \log(d/\beta)}{\diffp_0}  \right),$$ and $\bar L\opt$ is bounded by $L\opt + 5\log(T/\beta)/\diffp_0$ with high probability, this proves the claim.
\end{proof}
}

\iftoggle{arxiv}{}{We also present a different binary-tree based mechanism for this problem with similar rates in~\Cref{sec:bt-experts}.}

\iftoggle{arxiv}{
\subsection{A binary-tree based algorithm}

In this section, we present another algorithm which achieves the optimal regret for settings with zero-expert loss. Instead of using sparse-vector, this algorithm builds on the binary tree mechanism. The idea is to repetitively select $O(\mathsf{poly}(\log(dT)))$ random good experts and apply the binary tree to calculate a private version of their aggregate losses. Whenever all of the chosen experts are detected to have non-zero loss, we choose a new set of good experts. Similarly to~\cref{alg:SVT-zero-loss}, we can show that each new phase reduces the number of good experts by a constant factor as an oblivious adversary does not know the choices of the algorithm, hence there are only $O(\mathsf{poly}(\log(dT)))$ phases.

We provide a somewhat informal description of the algorithm in~\cref{alg:Bin-tree-zero-loss}. This algorithm also achieves regret $O(\mathsf{poly}(\log(dT))/\diffp)$ in the realizable case. We do not provide a proof as it is somewhat similar to that of~\cref{thm:ub-realizable}.
\begin{algorithm}
	\caption{Binary-tree algorithm for zero loss experts (sketch)}
	\label{alg:Bin-tree-zero-loss}
	\begin{algorithmic}[1]
		\STATE Set $k=0$ and $B=  O(\mathsf{poly}(\log(dT)))$
        \WHILE{$t \le T$\,}
            \STATE Use the exponential mechanism with score function $s(x) = \sum_{i=1}^t \ell_i(x)$ to privately select a set $S_k$ of $B$ experts from $[d] \setminus \cup_{0 \le i \le k} S_i$
            \STATE Apply binary tree for each expert $x \in S_k$ to get private aggregate estimates for $ \sum_{i=1}^t \ell_i(x)$ for every $t \in [T]$
            \STATE Let $\hat c_{t,x}$ denote the output of the binary tree for expert $x \in S_k$ at time $t$ 
            \WHILE{there exists $x \in S_k$ such that $\hat c_{t,x} \le O(\mathsf{poly}(\log(dT))/\diffp)$}
                \STATE Receive $\ell_t : [d] \to [0,1]$
                \STATE Choose $x_t \in S_k$ that minimizes $\hat c_{t,x}$
        	    \STATE Pay error $\ell_t(x_t)$
        	    \STATE $t = t+1$
            \ENDWHILE
            \STATE $k = k + 1$
        \ENDWHILE
	\end{algorithmic}
\end{algorithm}
}

%% file: realizable-ub-OCO.tex
\section{Faster rates for DP-OCO}
\label{sec:dp-oco}
In this section we study differentially private online convex optimization (DP-OCO) and propose new algorithms with faster rates in the realizable setting. In~\Cref{sec:dp-oco-experts}, we develop an algorithm that reduces the OCO problem to an experts problem (by discretizing the space) and then uses our procedure for experts. 
In~\Cref{sec:dp-oco-smooth}, we show that follow-the-regularized-leader (FTRL) using the binary tree mechanism results in faster rates in the realizable setting for smooth functions.

\subsection{Experts-based algorithm for DP-OCO}
\label{sec:dp-oco-experts}
The algorithm in this section essentially reduces the problem of DP-OCO to DP-OPE by discretizing the space $\mc{X} = \{ x \in \R^d: \ltwo{x} \le D \}$ into sufficiently many experts. In particular, we consider a $\rho$-net of the space $\mc{X}$, that is, a set $\mc{X}_{\mathsf{experts}} = \{x^1, \dots, x^M \} \subset \mc{X}$ such that for all $x \in \mc{X}$ there is $x^i \in \mc{X}^\rho_{\mathsf{experts}}$ such that $\ltwo{x^i - x} \le \rho$. Such a set exists if $M \ge 2^{d \log(4D/\rho)}$ (\citealp{Duchi19}, Lemma 7.6). Given a loss function $\ell_t: \mc{X} \to \R$, we define the loss of expert $x^i$ to be $\ell_t(x^i)$. Then, we run~\Cref{alg:SVT-zero-loss} for the given DP-OPE problem. This algorithm has the following guarantees.

\begin{theorem}
\label{thm:DP-OCO}
    Let $\mc{X} = \{ x \in \R^d: \ltwo{x} \le D\}$
    and $\ell_1,\dots,\ell_T : \mc{X} \to \R$ be non-negative, convex and $L$-Lipschitz functions chosen by an oblivious adversary.
    Then running~\Cref{alg:SVT-zero-loss}  over $\mc{X}^\rho_{\mathsf{experts}}$ with $\rho = 1/(LT)$ is \ed-DP and with probability at least $1-O(\beta)$ has regret 
    \iftoggle{arxiv}{
    \begin{equation*}
        O\left( L^{\opt} d \log(LD/\beta) + \frac{d^{3/2} \log^{3/2}(LDT) \sqrt{\log(1/\delta)} + \log(T/\beta) d \log(LD/\beta)}{\diffp}  \right).
    \end{equation*} 
    }
    {
    \begin{align*}
     & \E\left[ \sum_{t=1}^T \ell_t(x_t) - \min_{x \in \mc{X}} \sum_{t=1}^T \ell_t(x) \right] \\
     & \le (L\opt + \frac{1}{\diffp}) d^{1.5} \cdot O\left( \mathsf{poly} (\log (DLT/\delta)) \right).
    \end{align*} 
    }
\end{theorem}
\iftoggle{arxiv}{}{We defer the proof to~\Cref{sec:proof-dp-oco}.}
\iftoggle{arxiv}{
\begin{proof}
    Let $x_1,\dots,x_T$ be the experts chosen by the algorithm.
    First, \Cref{thm:ub-realizable} implies that this algorithm obtains the following regret with respect to the best expert
    \begin{equation*}
        \sum_{t=1}^T \ell_t(x_t) - L^{\opt}_{\mathsf{experts}} \le 
        O\left( L^{\opt}_{\mathsf{experts}} \log(M/\beta) + \frac{\log^2(M) + \log(T/\beta) \log(M/\beta)}{\diffp}  \right)
    \end{equation*}
    where $L^{\opt}_{\mathsf{experts}} = \min_{x \in \mc{X}^\rho_{\mathsf{experts}}} \sum_{t=1}^T \ell_t(x)$. Since $\ell_t$ is $L$-Lipschitz for each $t \in [T]$, we obtain that 
    \begin{equation*}
    | L\opt - L^{\opt}_{\mathsf{experts}} |
        = |\min_{x \in \mc{X}} \sum_{t=1}^T \ell_t(x) - \min_{x \in \mc{X}^\rho_{\mathsf{experts}}} \sum_{t=1}^T \ell_t(x) |  
        \le T L  \rho.
    \end{equation*}
    Overall this gives
     \begin{equation*}
        \sum_{t=1}^T \ell_t(x_t) - L^{\opt} 
        \le  O\left( (L^{\opt} + TL\rho) \log(M/\beta) + \frac{\log^{3/2}(M) \sqrt{\log(1/\delta)} + \log(T/\beta) \log(M/\beta)}{\diffp}  \right).
    \end{equation*}
    Setting $\rho = 1/(LT\diffp)$ proves the claim.
\end{proof}
}

These results demonstrates that existing algorithms which achieve normalized regret roughly $(\ifrac{\sqrt{d}}{T \diffp})^{1/2}$ are not optimal for the realizable setting. Moreover, in the low-dimensional regime (constant $d$), the above bound is nearly-optimal up to logarithmic factors as we have a lower bound of $\sqrt{d}/T\diffp$ from the stochastic setting of this problem (see discussion in the introduction).

Finally, while the algorithm we presented in~\Cref{thm:DP-OCO} has exponential runtime due to discretizing the space, we note that applying~\Cref{alg:SVT-zero-loss} over the unit ball results in similar rates and polynomial runtime.
Recall that this algorithm only accesses the loss functions to sample from the exponential mechanism, and uses sparse-vector over the running loss. Both of these can be implemented in polynomial time---since the losses are convex---using standard techniques from log-concave sampling.

\subsection{Binary-tree based FTRL}
\label{sec:dp-oco-smooth}
In this section, we consider DP-OCO with smooth loss functions and show that DP-FTRL~\cite[Algorithm 1]{KairouzMcSoShThXu21} with modified parameters obtains improved normalized regret ${\beta D^2}/{T} + ({\sqrt{d}}/{T \diffp})^{2/3}$ in the realizable setting, compared to ${LD}/{\sqrt{T}} + ({\sqrt{d}}/{T \diffp} )^{1/2}$ in the non-realizable setting.

We present the details in~\Cref{alg:dp-ftrl}. Appendix B.1 in~\cite{KairouzMcSoShThXu21} has more detailed information about the implementation of the binary tree mechanism in DP-FTRL.

\begin{algorithm}
	\caption{DP-FTRL~\cite{KairouzMcSoShThXu21}}
	\label{alg:dp-ftrl}
	\begin{algorithmic}[1]
	\REQUIRE Regularization parameter $\lambda$	
        \STATE Set $x_0 \in \mc{X}$
        \FOR{$t=1$ to $T$\,}
            \STATE Use the binary tree mechanism to estimate the sum $\sum_{i=1}^{t-1} \nabla \ell_i(x_i)$; let $\bar g_{t-1}$ be the estimate
            \STATE Apply follow-the-regularized-leader step
            \begin{equation*}
             x_{t} = \argmin_{x \in \mc{X}} \<\bar g_{t-1}, x \> + \frac{\lambda}{2} \ltwo{x}^2,
            \end{equation*}
            \STATE Receive loss function $\ell_t: \mc{X} \to \R$
            \STATE Pay cost $\ell_t(x_t)$
        \ENDFOR
	\end{algorithmic}
\end{algorithm}

We have the following guarantees for DP-FTRL in the realizable and smooth setting. 
\begin{theorem}
\label{thm:dp-oco-smooth}
    Let $\mc{X} = \{ x \in \R^d: \ltwo{x} \le D\}$
    and $\ell_1,\dots,\ell_T : \mc{X} \to \R$ be non-negative, convex, $L$-Lipschitz, and $\beta$-smooth functions chosen by an oblivious adversary.  DP-FTRL with $\lambda = 32 \beta + \left( \frac{\beta}{\diffp^2} (L/D)^2 T d \log(T) \log(1/\delta) \right)^{1/3}$ is \ed-DP and generates
    $x_1,\dots,x_T$ that has regret
    \iftoggle{arxiv}{
    \begin{align*}
    \frac{1}{T} \E \left[ \sum_{t=1}^T \ell_t(x_t) - \ell_t(x\opt) \right]
        & \le  O \left( \frac{L\opt +  \beta D^2}{T} +  \left(  LD \frac{ \sqrt{\beta D^2 d \log(T) \log(1/\delta)}}{T \diffp} \right)^{2/3}  \right).
    \end{align*}
    }
    {
    \begin{align*}
    & \frac{1}{T} \E \left[ \sum_{t=1}^T \ell_t(x_t) - \ell_t(x\opt) \right] \\
        & \le  O \left( \frac{L\opt +  \beta D^2}{T} +  \left(  LD \frac{ \sqrt{\beta D^2 d \log(T) \log(1/\delta)}}{T \diffp} \right)^{2/3}  \right).
    \end{align*}
    }
\end{theorem}
For the proof, we use the following property for smooth non-negative functions.
\begin{lemma}[\citealp{Nesterov04}]
    Let $\ell: \mc{X} \to \R$ be non-negative and $\beta$-smooth function. Then $\ltwo{\nabla \ell(x)}^2\le 4 \beta \ell(x)$.
\end{lemma}
\begin{proof}
    The proof follows similar arguments to the proof of Theorem 5.1 in~\cite{KairouzMcSoShThXu21}. Let 
    \begin{equation*}
        x_{t+1} = \argmin_{x \in \mc{X}} \sum_{i=1}^t \<\nabla \ell_i(x_i), x \> + \frac{\lambda}{2} \ltwo{x}^2 + \<b_t,x\>,
    \end{equation*}
    be the iteration of DP-FTRL where $b_t$ is the noise added by the binary tree mechanism. Moreover, let $\hat x_{t+1}$ be the non-private solution, that is, 
    \begin{equation*}
        \hat x_{t+1} = \argmin_{x \in \mc{X}} \sum_{i=1}^t \<\nabla \ell_i(x_i), x \> + \frac{\lambda}{2} \ltwo{x}^2 .
    \end{equation*}
    Lemma C.2 in~\cite{KairouzMcSoShThXu21} states that $\ltwo{x_{t+1} - \hat x_{t+1}} \le \ltwo{b_t}/\lambda$. Therefore, we have
    \iftoggle{arxiv}{
    \begin{align*}
    \sum_{t=1}^T \ell_t(x_t) - \ell_t(x\opt) 
        & \le \sum_{t=1}^T  \< \nabla \ell_t(x_t), x_t - x\opt\> \\
        & = \sum_{t=1}^T  \< \nabla \ell_t(x_t), x_t - \hat x_t\> + \sum_{t=1}^T  \< \nabla \ell_t(x_t), \hat x_t - x\opt\> \\
        & \le \sum_{t=1}^T \ltwo{\nabla \ell_t(x_t)} \ltwo{x_t - \hat x_t} + \sum_{t=1}^T  \< \nabla \ell_t(x_t), \hat x_t - x\opt\> \\ 
        & \le \frac{1}{8 \beta} \sum_{t=1}^T \ltwo{\nabla \ell_t(x_t)}^2 +  4\beta \sum_{t=1}^T\ltwo{x_t - \hat x_t}^2 + \sum_{t=1}^T  \< \nabla \ell_t(x_t), \hat x_t - x\opt\> \\ 
        & \le \frac{1}{2} \sum_{t=1}^T \ell_t(x_t) +  4 \beta \sum_{t=1}^T \ltwo{b_t}^2/\lambda^2 + \sum_{t=1}^T  \< \nabla \ell_t(x_t), \hat x_t - x\opt\>,
    \end{align*}
    }
    {
    \begin{align*}
    & \sum_{t=1}^T \ell_t(x_t) - \ell_t(x\opt) \\
        & \le \sum_{t=1}^T  \< \nabla \ell_t(x_t), x_t - x\opt\> \\
        & = \sum_{t=1}^T  \< \nabla \ell_t(x_t), x_t - \hat x_t\> + \sum_{t=1}^T  \< \nabla \ell_t(x_t), \hat x_t - x\opt\> \\
        & \le \sum_{t=1}^T \ltwo{\nabla \ell_t(x_t)} \ltwo{x_t - \hat x_t} + \sum_{t=1}^T  \< \nabla \ell_t(x_t), \hat x_t - x\opt\> \\ 
        & \le \frac{1}{8 \beta} \sum_{t=1}^T \ltwo{\nabla \ell_t(x_t)}^2 +  4\beta \sum_{t=1}^T\ltwo{x_t - \hat x_t}^2 + \sum_{t=1}^T  \< \nabla \ell_t(x_t), \hat x_t - x\opt\> \\ 
        & \le \frac{1}{2} \sum_{t=1}^T \ell_t(x_t) +  4 \beta \sum_{t=1}^T \ltwo{b_t}^2/\lambda^2 + \sum_{t=1}^T  \< \nabla \ell_t(x_t), \hat x_t - x\opt\>,
    \end{align*}
    
    }
    where the second inequality follows from the Fenchel-Young inequality.
    We can now upper bound the right term. Indeed, Theorem 5.2 in~\cite{Hazan16} implies that FTRL has 
    \begin{align*}
        \sum_{t=1}^T  \< \nabla \ell_t(x_t), \hat x_t - x\opt\> & \le \frac{2}{\lambda} \sum_{t=1}^T \ltwo{\nabla \ell_t(x_t)}^2 + \lambda D^2 \\
        & \le \frac{8 \beta}{\lambda} \sum_{t=1}^T \ell_t( x_t)  + \lambda D^2.
    \end{align*}
    Overall we now get 
    \iftoggle{arxiv}{
    \begin{align*}
    \sum_{t=1}^T \ell_t(x_t) - \ell_t(x\opt) 
        & \le \frac{1}{2} \sum_{t=1}^T \ell_t(x_t) +  \frac{4\beta }{ \lambda^2} \sum_{t=1}^T \ltwo{b_t}^2  
        + \frac{8 \beta}{\lambda} \sum_{t=1}^T \ell_t( x_t)  + \lambda D^2.
    \end{align*}  
    }
    {
    \begin{align*}
    \sum_{t=1}^T \ell_t(x_t) - \ell_t(x\opt) 
        & \le \frac{1}{2} \sum_{t=1}^T \ell_t(x_t) +  \frac{4\beta }{ \lambda^2} \sum_{t=1}^T \ltwo{b_t}^2  \\
        & + \frac{8 \beta}{\lambda} \sum_{t=1}^T \ell_t( x_t)  + \lambda D^2.
    \end{align*}  
    
    }
    The binary tree mechanism also guarantees that for all $t \in [T]$, $$\E[\ltwo{b_t}^2] \le O \left( \frac{L^2 d \log(T) \log(1/\delta)}{\diffp^2} \right)$$ (see Appendix B.1 in~\cite{KairouzMcSoShThXu21}).
    Thus, taking expectation and setting the regularization parameter to $\lambda = 32 \beta + \big( \frac{\beta}{\diffp^2} (L/D)^2 T d \log(T) \log(1/\delta) \big)^{1/3}$, we have
    \iftoggle{arxiv}{
    \begin{align*}
    \E \left[\sum_{t=1}^T \ell_t(x_t) - \ell_t(x\opt) \right]
        & \le  O \left(L\opt +  \beta D^2 +  \left( \beta D^2 (LD)^2 \frac{T d \log(T) \log(1/\delta)}{\diffp^2} \right)^{1/3}  \right).
    \end{align*} 
    }
    {
    \begin{align*}
    & \E \left[\sum_{t=1}^T \ell_t(x_t) - \ell_t(x\opt) \right] 
         \le  O (L\opt +  \beta D^2 ) \\
         & \quad + O\left( \left( \beta D^2 (LD)^2 \frac{T d \log(T) \log(1/\delta)}{\diffp^2} \right)^{1/3}  \right).
    \end{align*} 
    
    }
\end{proof}

%% file: realizable-LB.tex
\section{Lower bounds}
\label{sec:real-LB}
In this section, we prove lower bounds for private experts in the realizable setting which show that our upper bounds are nearly-optimal up to logarithmic factors. The lower bound demonstrates that a logarithmic dependence on $d$ is necessary even in the realizable setting. Note that for DP-OCO in the realizable setting, a lower bound of $d/T\diffp$ for pure DP follows from known lower bounds for DP-SCO in the interpolation regime~\cite{AsiChChDu22} using online-to-batch conversions~\cite{Hazan16}.

The following theorem states our lower bound for DP-OPE.
\begin{theorem}
\label{thm:lb-obl-experts}
    Let $\diffp \le 1/10$ and $\delta \le \diffp/d$.
    If $\A$ is $(\diffp,\delta)$-DP then there is an oblivious adversary such that $\min_{x\in[d]} \sum_{t=1}^T \ell_t(x) = 0$ and 
    \begin{equation*}
        \E\left[\sum_{t=1}^T \ell_t(x_t) - \min_{x \in [d]} \sum_{t=1}^T \ell_t(x)\right]
        \ge \Omega \left( \frac{\log(d)}{\diffp} \right).
    \end{equation*}
\end{theorem}
\begin{proof}
    Let $\ell^0(x)=0$ for all $x$ and for $j \in [d]$
    let $\ell^j(x)$ be the function that has $\ell^j(x)=0$ for $x=j$ and otherwise $\ell^j(x)=1$. The oblivious adversary picks one of the following $d$ sequences uniformly at random: $\Ds^j = (\underbrace{\ell^0,\dots,\ell^0}_{T-k},\underbrace{\ell^j,\dots,\ell^j}_{k})$ where $k = \frac{\log d}{2 \diffp}$ and $j \in [d]$. Assume towards a contradiction that the algorithm obtains regret $\log(d)/(32\diffp)$. This implies that there exists $d/2$ sequences such that the algorithm obtains expected regret $\log(d)/(16\diffp)$ where the expectation is only over the randomness of the algorithm. Assume without loss of generality these sequences are $\Ds^1,\dots,\Ds^{d/2}$. Let $B_j$ be the set of outputs that has low regret on $\Ds^j$, that is,
    \begin{equation*}
        B_j = \{ (x_1,\dots,x_T) \in [d]^T: \sum_{t=1}^T \ell^j(x_t) \le \log(d)/(8\diffp) \}.
    \end{equation*}
    Note that $B_j \cap B_{j'} = \emptyset$ since 
    if $x_{1:T} \in B_j$ then at least $3k/4 = 3\log(d)/(8\diffp)$ of the last $k$ outputs must be equal to $j$. Now Markov inequality implies that 
    \begin{equation*}
        \P(\A(\Ds^j) \in B_j) \ge 1/2.
    \end{equation*}
    Moreover, group privacy gives
    \begin{align*}
    \P(\A(\Ds^{j}) \in B_{j'}) 
        & \ge e^{-k\diffp} \P(\A(\Ds^{j'}) \in B_{j'}) - k e^{-\diffp} \delta \\
        & \ge \frac{1}{2 \sqrt{d}} - \frac{\log(d)}{2\diffp} \delta \\
        & \ge \frac{1}{4 \sqrt{d}},
    \end{align*}
    where the last inequality follows since $\delta \le \diffp/d$.
    Overall we get that 
    \begin{align*}
        \frac{d/2-1}{4\sqrt{d}}  \le 
        \P(\A(\Ds^{j}) \notin B_{j}) 
         \le \frac{1}{2}  ,
    \end{align*}
    which is a contradiction for $d \ge 32$.
\end{proof}

%% file: conclusion.tex
\section{Conclusion}
In this work, we studied differentially private online learning problems in the realizable setting, and developed algorithms with improved rates compared to the non-realizable setting. However, several questions remain open in this domain. First, our near-optimal algorithms for DP-OPE obtain $\log^{1.5}(d)/\diffp$ regret, whereas the lower bound we have is $\Omega(\log(d)/\diffp)$. Hence, perhaps there are better algorithms with tighter logarithmic factors than our sparse-vector based algorithms. Additionally, for DP-OCO, our algorithms are optimal only for low-dimensional setting, and there remains polynomial gaps in the high-dimensional setting. Finally, optimal rates for both problems (DP-OPE and DP-OCO) are still unknown in the general non-realizable setting.

%% file: appendix-chernoff.tex
\section{Concentration for sums of geometric variables}

In this section, we proof a concentration result for the sum of geometric random variables, which allows us to upper bound the number of switches in the sparse-vector based algorithm.
We say that $Z$ is geometric random variable with success probability $p$ if $P(W=k) = (1-p)^{k-1}p$ for $k\in\{1,2,\dots\}$. To this end, we use the following Chernoff bound.
\begin{lemma}[\citealp{MitzenmacherUp05}, Ch.~4.2.1]
  \label{lemma:chernoff}
  Let $X = \sum_{i=1}^n X_i$ for $X_i \simiid \mathsf{Ber}(p)$.
  Then for $\delta \in [0,1]$,
  \begin{align*}
    \P(X > (1+\delta)np ) \le e^{-np\delta^2 /3}
    ~~~ \mbox{and} ~~~
    \P(X < (1-\delta)np ) \le e^{-np\delta^2 /2}.
  \end{align*}
\end{lemma}

The following lemma demonstrates that the sum of geometric random variables concentrates around its mean with high probability.
\begin{lemma}
\label{lemma:geom-concentration}
Let $W_1,\dots,W_n$ be iid geometric random variables with success probability $p$. Let $W = \sum_{i=1}^n W_i$. Then for any $k \ge n$ 
\begin{equation*}
    \P(W > 2k/p ) \le \exp{\left(-k/4\right)}.
\end{equation*}
\end{lemma}
\begin{proof}
    Notice that $W$ is distributed according to the negative binomial distribution where we can think of $W$ as the number of Bernoulli trials until we get $n$ successes. More precisely, let $\{B_i\}$ for $i\ge1$ be Bernoulli random variables with probability $p$. Then the event $W>t$ has the same probability as $\sum_{i=1}^t B_i < n$. Thus we have that 
    \begin{equation*}
        \P(W > t ) \le \P(\sum_{i=1}^t B_i < n).
    \end{equation*}
    We can now use Chernoff inequality (\Cref{lemma:chernoff}) to get that for $t = 2n/p$
    \begin{align*}
    \P(\sum_{i=1}^t B_i < n) 
    \le \exp{(-tp/8)} = \exp{(-n/4)}.
    \end{align*}
    This proves that 
    \begin{equation*}
    \P(W > 2n/p ) \le \exp{\left(-n/4\right)}.
    \end{equation*}
    The claim now follows by noticing that $\sum_{i=1}^n W_i \le \sum_{i=1}^k W_i $ for $W_i$ iid geometric random variable when $k \ge n$, thus $\P(\sum_{i=1}^n W_i \ge 2k/p) \le \P(\sum_{i=1}^k W_i \ge 2k/p) \le \exp{\left(-k/4\right)}$

\end{proof}

%% file: bib.bib
@string{colt12 = {Proceedings of the Twenty Fifth Annual Conference on
		  Computational Learning Theory}}

@string{nips2013= {Advances in Neural Information Processing Systems 26}}

@string{icml14 = {Proceedings of the 31st International Conference on Machine Learning}}

@string{icml17 = {Proceedings of the 34th International Conference on Machine Learning}}

@string{focs17 = {58th Annual Symposium on Foundations of Computer Science}}

@inproceedings{AgarwalSi17,
  title={The price of differential privacy for online learning},
  author={Naman Agarwal and Karan Singh},
  booktitle= icml17,
  pages={32--40},
  year={2017},
}

@misc{Duchi19,
author = {John C. Duchi},
title = {Information Theory and Statistics},
year = 2019,
howpublished = {Lecture Notes for Statistics 311/{EE} 377,
Stanford University},
note = {Accessed May 2019},
url = {http://web.stanford.edu/class/stats311/lecture-notes.pdf},
}

@inproceedings{DworkNaPiRo10,
author = {Cynthia Dwork and Moni Naor and Toniann Pitassi and Guy N Rothblum},
title = {Differential privacy under continual observation},
year = 2010,
booktitle = {Proceedings of the Forty-Second Annual ACM
		  Symposium on the Theory of Computing},
pages = {715--724},
}

@article{DworkRo14,
 author = {Dwork, Cynthia and Roth, Aaron},
 title = {The Algorithmic Foundations of Differential Privacy},
 journal = {Foundations and Trends in Theoretical Computer Science},
 volume = {9},
 number = {3 \& 4},
 year = {2014},
 pages = {211--407},
 numpages = {197},
 publisher = {Now Publishers Inc.},
 address = {Hanover, MA, USA},
}

@inproceedings{FeldmanKoTa20,
  title={Private stochastic convex optimization: optimal rates in linear time},
  author={Vitaly Feldman and Tomer Koren and Kunal Talwar},
  booktitle={Proceedings of the 52nd Annual ACM on the Theory of Computing},
  pages={439--449},
  year={2020}
}

@book{MitzenmacherUp05,
	title={Probability and computing: Randomized algorithms and probabilistic 
	analysis},
	author={Mitzenmacher, Michael and Upfal, Eli},
	year={2005},
	publisher={Cambridge University Press}
}

@inproceedings{SmithTh13,
author = {Adam Smith and Abhradeep Thakurta},
title = {({N}early) optimal algorithms for private online
  learning in full-information and bandit settings},
year = 2013,
booktitle = nips2013,
}

@inproceedings{SteinkeUll17b,
  title={Tight lower bounds for differentially private selection},
  author={Thomas Steinke and Jonathan Ullman},
  booktitle=focs17,
  pages={552--563},
  year={2017},
  organization={IEEE}
}

@article{KairouzMcSoShThXu21,
  title={Practical and Private (Deep) Learning without Sampling or Shuffling},
  author={Peter Kairouz and Brendan McMahan and Shuang Song and Om Thakkar  and Abhradeep Thakurta and Zheng Xu},
  journal={arXiv:2103.00039 [cs.CR]},
  year={2021}
}

@inproceedings{JainKoTh12,
  title = {Differentially private online learning},
  author = {Prateek Jain and Pravesh Kothari and Abhradeep Thakurta},
  booktitle = colt12,
  year={2012}
}

@inproceedings{JainTh14,
  title={({N}ear) dimension independent risk bounds for differentially private learning},
  author={Prateek Jain and Abhradeep Thakurta},
  booktitle= icml14,
  pages={476--484},
  year={2014}
}

@article{KalaiVe05,
author = {A. Kalai and S. Vempala},
title = {Efficient algorithms for online decision problems},
year = 2005,
journal = {Journal of Computer and System Sciences},
volume = 71,
number = 3,
pages = {291--307},
}

@article{Hazan16,
author = {Elad Hazan},
title = {Introduction to Online Convex Optimization},
year = 2016,
journal = {Foundations and Trends in Optimization},
volume = {2},
number = {3--4},
pages = {157--325},
}

@article{JainRaSiSm21,
  title={The Price of Differential Privacy under Continual Observation},
  author={Palak Jain and Sofya Raskhodnikova and Satchit Sivakumar and Adam Smith},
  journal={arXiv:2112.00828 [cs.DS]},
  year= 2021
}

@article{ChanShSo11,
  title={Private and continual release of statistics},
  author={T-H Hubert Chan  and Elaine Shi and Dawn Song},
  journal={ACM Transactions on Information and System Security (TISSEC)},
  volume={14},
  number={3},
  pages={1--24},
  year= 2011 ,
}

@InProceedings{AsiChChDu22,
  title = 	 {Private optimization in the interpolation regime: faster rates and hardness results},
  author =       {Hilal Asi and Karan Chadha and Gary Cheng and John Duchi},
  booktitle = 	 {Proceedings of the 39th International Conference on Machine Learning},
  year = 	 {2022},
}

@article{Shalev12,
  title={Online learning and online convex optimization},
  author={Shalev-Shwartz, Shai},
  journal={Foundations and Trends in Machine Learning},
  volume=4,
  number=2,
  pages={107--194},
  year=2012,
}

@inproceedings{SrebroSrTe10,
  title={Smoothness, low noise and fast rates},
  author={Srebro, Nathan and Sridharan, Karthik and Tewari, Ambuj},
  booktitle={nips2010},
  pages={2199--2207},
  year=2010
}

@book{Nesterov04,
author = {Y. Nesterov},
title = {Introductory Lectures on Convex Optimization},
publisher = {Kluwer Academic Publishers},
year = {2004},
}

@article{AsiFeKoTa22,
author  = {Hilal Asi and Vitaly Feldman and Tomer Koren and Kunal Talwar},
title =	{Private Online Prediction from Experts: Separations and Faster Rates
},
year = 2022,
journal = {	arXiv:2210.13537 [cs.LG]},
}
